\documentclass[twocolumn]{article}

\usepackage{graphicx} 
\usepackage[centerlast]{subfigure} 

\usepackage{natbib}

\usepackage{algorithm}
\usepackage{algorithmic}

\usepackage{hyperref}

\usepackage{amssymb}
\usepackage{amsfonts}
\usepackage{amsmath}
\usepackage{amsthm}
\def\DP{\mbox{DP}}
\def\GaP{\mbox{GaP}}
\def\Dir{\mbox{Dirichlet}}
\def\Gam{\mbox{Gamma}}

\newtheorem{theorem}{Theorem}

\begin{document} 

\title{Exact and Efficient Parallel Inference for Nonparametric Mixture Models}
\author{Sinead A. Williamson \and Avinava Dubey \and Eric P. Xing}
\date{}
\maketitle
\begin{abstract} 
Nonparametric mixture models based on the Dirichlet process are an elegant alternative to finite models when the number of underlying components is unknown, but inference in such models can be slow. Existing attempts to parallelize inference in such models have relied on introducing approximations, which can lead to inaccuracies in the posterior estimate. In this paper, we describe auxiliary variable representations for the Dirichlet process and the hierarchical Dirichlet process that allow us to sample from the true posterior in a distributed manner. We show that our approach allows scalable inference without the deterioration in estimate quality that accompanies existing methods.
\end{abstract} 
\section{Introduction}


Models based on the Dirichlet process \citep[DP,][]{Ferguson:1973} and its extension the hierarchical Dirichlet process \citep[HDP,][]{Teh:Jordan:Beal:Blei:2006} have a number of appealing properties. They allow a countably infinite number of mixture components  {\it a priori}, meaning that a finite dataset will be modeled using a finite, but random, number of parameters.  If we observe more data, the model can grow in a consistent manner. Unfortunately, while this means that such models can theoretically cope with data sets of arbitrary size and latent dimensionality, in practice inference can be slow, and the memory requirements are high. 

Parallelization is a technique often used to speed up computation, by splitting the computational and memory requirements of an algorithm onto multiple machines. Parallelization of an algorithm involves exploitation of (conditional) independencies. If we can update one part of a model independently of another part, we can split the corresponding sections of code onto separate processors or cores. Unfortunately, many models do not have appropriate independence structure, making parallelization difficult. For example, in the Chinese restaurant process representation of a Dirichlet process mixture model, the conditional distribution over the cluster allocation of a single data point depends on the allocations of \emph{all} the other data points.

In such cases, a typical approach is to apply approximations that break some of the long-range dependencies. While this can allow us to parallelize inference in the approximate model, the posterior estimate will, in general, be less accurate. Another option is to use a sequential Monte Carlo approach, where the posterior is approximated with a swarm of independent particles. In its simplest form, this approach is inherently parallelizable, but such a naive implementation will run into problems of variance overestimation. We can improve the estimate quality by introducing global steps such as particle resampling and Gibbs steps, but these steps cannot easily be parallelized.

In this paper, we show how the introduction of auxiliary variables into the DP and HDP can create the conditional independence structure required to obtain a parallel Gibbs sampler, without introducing approximations. As a result, we can make use of the powerful and elegant representations provided by the DP and HDP, while maintaining manageable computational requirements. We show that the resulting samplers are able to achieve significant speed-ups over existing inference schemes for the ``exact'' models, with no loss in quality. By performing inference in the true model, we are able to achieve better results than those obtained using approximate models.

\section{Background}

The Dirichlet process  is a distribution over discrete probability measures $D=\sum_{k=1}^\infty \pi_k\delta_{\theta_k}$ with countably infinite support, where the finite-dimensional marginals are distributed according to a finite Dirichlet distribution. It is parametrized by a base probability measure $H$, which determines the distribution of the atom locations, and a concentration parameter $\alpha>0$, which acts like an inverse variance. The DP can be used as the distribution over mixing measures in a nonparametric mixture model. In the Dirichlet process mixture model \citep[DPMM,][]{Antoniak:1974}, data $\{x_i\}_{i=1}^n$ are assumed to be generated according to
\begin{equation}
\begin{array}{ccccc}
D\sim\DP(\alpha, H)\,,& &
\theta_i \sim D\,, & &
x_i\sim f(\theta_i)\,.
\end{array}\label{eq:DP}
\end{equation}
While the DP allows an infinite number of clusters {\it a priori}, any finite dataset will be modeled using a finite, but random, number of clusters.

Hierarchical Dirichlet processes extend the DP to model grouped data. The HDP is a distribution over probability distributions $D_m, m=1,\dots,M$, each of which is marginally distributed according to a DP. These distributions are coupled using a discrete common base-measure, itself distributed according to a DP. Each distribution $D_m$ can be used to model a collection of observations $\mathbf{x}_m := \{x_{mi}\}_{i=1}^{N_m}$. 
\begin{equation}
\begin{array}{cc}
D_0\sim\DP(\alpha, H)&
D_m\sim\DP(\gamma,D_0)\\
\theta_{mi} \sim D_m &
x_{mi}\sim f(\theta_{mi})\,
\end{array}\label{eq:HDP}
\end{equation}
for $m=1,\dots, M$ and $i=1,\dots, N_m$. HDPs have been used to model data including text corpora \cite{Teh:Jordan:Beal:Blei:2006}, images \cite{Sudderth:Torralba:Freeman:Willsky:2005}, time series data \cite{Fox:Sudderth:Jordan:Willsky:2008}, and genetic variation \cite{Sohn:Xing:2009}.


A number of inference schemes have been developed for the DP and the HDP. The most commonly used class of inference methods is based on the Chinese restaurant process (CRP, see for example \citet{Aldous:1983}). Such schemes integrate out the random measures ($D$ in Eq.~\ref{eq:DP}; $D_0$ and $\{D_m\}_{m=1}^M$ in Eq.~\ref{eq:HDP}) to obtain the conditional distribution for the cluster allocation of a single data point, conditioned on the allocations of all the other data points. A variety of Gibbs samplers based on the CRP can be constructed for the DP \cite{Neal:1998} and the HDP \cite{Teh:Jordan:Beal:Blei:2006}.
Unfortunately, because the conditional distribution for the cluster allocation of a single data point depends on all the data, this step cannot be parallelized without introducing approximations.

An alternative class of inference schemes involve explicitly instantiating the random measures \cite{Ishwaran:James:2001}. In this case, the observations are i.i.d. given the random measures, and can be sampled in parallel. However, since the random measure depends on the cluster allocations of all the data points, sampling the random measure cannot be parallelized.

Among existing practical schemes for parallelizable inference in DP and HDP, the following three are the most popular:


\subsection{Sequential Monte Carlo methods}
Sequential Monte Carlo (SMC) methods approximate a distribution of interest using a swarm of weighted, sequentially updated, particles. SMC methods have been used to perform inference in a number of models based on the DP \cite{Fearnhead:2004,Ulker:Gunsel:Cemgil:2010,Rodriguez:2011,Ahmed:Ho:Teo:Eisenstein:Smola:Xing:2011}. Such methods are appealing when considering parallelization, because each particle (and its weight) are updated independently of the other particles, and need consider only one data point at a time. However, a naive implementation where each particle is propagated in isolation leads to very high variance in the resulting estimate. To avoid this, it is typical to introduce \emph{resampling} steps, where the current swarm of particles is replaced by a new swarm sampled from the current posterior estimate. This avoids an explosion in the variance of our estimate, but the resampling cannot be performed in parallel. 

\subsection{Variational inference}
Variational Bayesian inference algorithms have been developed for both the DP \cite{Blei:Jordan:2004, Kurihara:Welling:Teh:2007} and the HDP \cite{Teh:Kurihari:Welling:2007, Wang:Paisley:Blei:2011}. Variational methods approximate a posterior distribution $p(\theta|X)$ with a distribution $q(\theta)$ belonging to a more manageable family of distributions and try to find the ``best'' member of this family, typically by minimizing the Kullback-Leibler divergence between $p(\theta|X)$ and $q(\theta)$. This also gives us a lower bound on the log likelihood, $\log p(X)$. A typical approach to selecting the family of approximating distributions is to assume independencies that may not be present in the true posterior. This means that variational algorithms are often easy to parallelize. However, by searching only within a restricted class of models we lose some of the expressiveness of the model, and will typically obtain less accurate results than MCMC methods that sample from the true posterior.

\subsection{Approximate parallel Gibbs sampling}
An approximate distributed Gibbs sampler for the HDP is described by \citet{Asuncion:Smyth:Welling:2008}. The basic sampler alternates distributed Gibbs steps with a global synchronization step. If we have $P$ processors, in the distributed Gibbs steps, each processor updates $1/P$ of the topic allocations. To sample the topic allocation for a given token, rather than use the current topic allocations of \emph{all} the other tokens, we use the current topic allocations for the tokens on the same processor, and for all other tokens we use the allocations after the last synchronization step. In the synchronization step, we amalgamate the topic counts. This can run into problems with topic alignment. In particular, there is no clear way to decide whether to merge a new topic on processor $p$ with a new topic on processor $p'$. An asynchronous version of the algorithm avoids the bottleneck of a global synchronization step.

\section{Introducing auxiliary variables to obtain conditional independence}
The key to developing a parallel inference algorithm is to exploit or introduce independencies. In the sequential Monte Carlo samplers, this independence can lead to high variance in the resulting estimate. In the other algorithms described, the independence is only achieved by introducing approximations.

If observations are modeled using a mixture model, then conditioned on the cluster allocations the observations are independent. The key idea that allows us to introduce conditional independence is the fact that, for appropriate parameter settings, \emph{Dirichlet mixtures of Dirichlet processes are Dirichlet processes}. In Theorems~\ref{thm:DP} and \ref{thm:HDP}, we demonstrate situations where this result holds, and develop mixtures of nonparametric models with the appropriate marginal distributions and conditional independence structures. The resulting models exhibit conditional independence between parameters that are coupled in Eq.~\ref{eq:DP} and Eq.~\ref{eq:HDP}, allowing us to perform parallel inference in Section~\ref{sec:inference} without resorting to approximations.

\begin{theorem}[Auxiliary variable representation for the DPMM]
We can re-write the generative process for a DPMM (given in Eq.~\ref{eq:DP}) as
\begin{equation}
\begin{split}
D_j\sim\DP\left(\frac{\alpha}{P}, H\right), \;\;\;\phi \sim \Dir\left(\frac{\alpha}{P},\dots,\frac{\alpha}{P}\right),\\
\pi_i\sim\phi,\qquad
\theta_i \sim D_{\pi_i}, \qquad
x_i\sim f(\theta_i),\qquad
\end{split}
\label{eq:DPaux}
\end{equation}
for $j=1,\dots,P$ and $i=1,\dots,N$.
The marginal distribution over the $x_i$ remains the same.\label{thm:DP}
\end{theorem}

\begin{proof}
We prove a more general result, that if $\phi\sim\Dir(\alpha_1,...,\alpha_P)$ and $D_j \sim \DP(\alpha_j, H_j)$, then $D := \sum_j\phi_jD_j\sim \DP(\sum_j\alpha_j,  \frac{\sum_j\alpha_jH_j}{\sum_j\alpha_j})$. 

A Dirichlet process with concentration parameter $\alpha$ and base probability measure $H$ can be obtained by normalizing a gamma process with base measure $\alpha H$. Gamma processes are examples of completely random measures \cite{Kingman:1967}, and it is well known that the superposition of $P$ completely random measures is again a completely random measure. In particular, if $G_j \sim \GaP(\alpha_j H_j)$, $j=1,\dots, P$, then $G := \sum_j G_j \sim \GaP(\sum_j\alpha_jH_j)$. 

Note that the total masses of the $G_j$ are gamma-distributed, and therefore the vector of normalized masses is Dirichlet-distributed. It follows that, if $\phi\sim\Dir(\alpha_1,...,\alpha_P)$ and $D_j \sim \DP(\alpha_j, H_j)$, then $D := \sum_j\phi_jD_j\sim \DP(\sum_j\alpha_j,\frac{\sum_j\alpha_j H_k}{\sum_j\alpha_j})$. This result, which is well known in the nonparametric Bayes community, is explored in more depth in Chapter 3 of \citet{Ghosh:Ramamoorth:2003}. An alternative proof is given in Appendix~\ref{app:proofs}.
\end{proof}

The auxiliary variables $\pi_i$ introduced in Eq.~\ref{eq:DPaux} introduce conditional independence, which we exploit to develop a distributed inference scheme. If we have $P$ processors, then we can split our data among them according to the $\pi_i$. Conditioned on their processor allocations, the data points are distributed according to independent Dirichlet processes. In Section~\ref{sec:inference}, we will see how we can combine local inference in these independent Dirichlet processes with global steps to move data between processors.

We can follow a similar approach with the HDP, assigning each token in each data point to one of $P$ groups corresponding to $P$ processors. However, to ensure the higher level DP can be split in a manner consistent with the lower level DP, we must impose some additional structure into the generative process described by Eq.~\ref{eq:HDP} -- specifically, that $\gamma \sim \Gam(\alpha)$. We can then introduce auxiliary variables as follows:

\begin{theorem}[Auxiliary variable representation for the HDP]
If we incorporate the requirement that the concentration parameter $\gamma$ for the bottom level DPs $\{D_j\}_{j=1}^M$ depends on the concentration parameter $\alpha$ for the top level DP $D_0$ as $\gamma\sim\Gam(\alpha)$, then we can rewrite the generative process for the HDP as:
\begin{equation}
\begin{alignedat}{4}
\zeta_j &\sim \Gam(\alpha / P) &\qquad \qquad &\pi_{mi}\sim& \nu_m \quad \;\;\\
D_{0j} &\sim \DP(\alpha/P, H) &\qquad \qquad &\theta_{mi} \sim& D_{m\pi_{mi}}\\
\nu_m &\sim \Dir(\zeta_1,\dots,\zeta_P) &\qquad \qquad &x_{mi}\sim& f(\theta_{mi})\\
D_{mj} &\sim \DP(\zeta_j,D_{0j}), & & &
\label{eq:HDPaux}
\end{alignedat}
\end{equation}
for $j=1,\dots,P$, $m=1,\dots,M$, and $i= 1,\dots,N_m$. The marginal distribution over the $\mathbf{x}_{mi}$ is the same in Eq.s~\ref{eq:HDP} and \ref{eq:HDPaux}.\label{thm:HDP}
\end{theorem}

\begin{proof}
If $\zeta_j \sim \Gam(\alpha/P)$, then $\gamma:=\sum_j\zeta_j \sim \Gam(\alpha)$. Since $D_{0j} \sim\DP(\alpha/P,H)$, it follows that $G_{0j}:=\zeta_jD_{0j} \sim \GaP\left(\frac{\alpha}{P}H\right)$ and
$G_0:=\sum_jG_{0j} \sim \GaP(\alpha H)$. If we normalize $G_0$, we find that $D_0:=\sum_j\frac{\zeta_j}{\gamma}D_{0j} \sim \DP(\alpha,H)$, as required by the HDP.

If we write $G_{mj}\sim \GaP(G_{0j})=\GaP(\zeta_j D_{0j})$, then we can see that $G_{mj} = \eta_{mj}D_{mj}$, where $\eta_{mj}\sim \Gam(\zeta_j)$, and $D_{mj}\sim\DP(\zeta_j, D_{0j})$.

If we normalize the $G_{mj}$, we find that 
\begin{equation*}
\begin{split}
D_m:= &\sum_j\frac{\eta_{mj}}{\sum_k\eta_{mk}}D_{mj}\\
 \sim&\DP\bigg(\sum_j\zeta_j,\frac{\sum_j\zeta_jD_{0j}}{\sum_j\zeta_j}\bigg)= \DP(\gamma, D_{0}),
\end{split}
\end{equation*}
as required by the HDP. Since the $\eta_{mj}$ only appear as a normalized vector, we can write $\nu_m = (\eta_{m1},\dots,\eta_{mP})/\sum_{j=1}^P \eta_{mj} \sim \Dir(\zeta_1,\dots,\zeta_P)$.

A more in-depth version of this proof is given in Appendix~\ref{app:proofs}.
\end{proof}

Again, the application to distributed inference is clear: Conditioned on the $\pi_{mi}$ we can split our data into $P$ independent HDPs. 

\section{Inference}\label{sec:inference}
The auxiliary variable representation for the DP introduced in Theorem~\ref{thm:DP}  makes the cluster allocations for data points where $\pi_i = j$ conditionally independent of the cluster allocations for data points where $\pi_i \neq j$. We can therefore split the data onto $P$ parallel processors or cores, based on the values of $\pi_i$. We will henceforth call $\pi_i$ the ``processor indicator'' for the $i$th data point. We can Gibbs sample the cluster allocations on each processor independently, intermittently moving data between clusters to ensure mixing of the sampler.

\subsection{Parallel inference in the Dirichlet process}\label{sec:DPinf}
We consider first the Dirichlet process. Under the construction described in Eq.~\ref{eq:DPaux}, each data point $x_i$ is associated with a processor indicator $\pi_i$ and a cluster indicator $z_i$. All data points associated with a single cluster will have the same processor indicator, meaning that we can assign each cluster to one of the $P$ processors (i.e., all data points in a single cluster are assigned to the same processor). Note that the $j$th processor will typically be associated with multiple clusters, corresponding to the local Dirichlet process $D_j$.  Conditioned on the assignments of the auxiliary variables $\pi_i$, the data points $x_i$ in Eq.~\ref{eq:DPaux} depend only on the local Dirichlet process $D_j$ and the associated parameters. 

We can easily marginalize out the $D_j$ and $\phi$. Assume that each data point $x_i$ is assigned to a processor $\pi_i\in\{1,\dots,P\}$, and a cluster $z_i$ residing on that processor (i.e., all data points in a single cluster are assigned to the same processor). We will perform \emph{local} inference on the cluster assignments $z_i$, and intermittently we will perform \emph{global} inference on the $\pi_i$.
\subsubsection{Local inference: Sampling the $z_i$}
Conditioned on the processor assignments, sampling the cluster assignments proceeds exactly as in a normal Dirichlet process with concentration parameter $\alpha / P$. A number of approaches for Gibbs sampling in the DPMM are described by \citet{Neal:1998}. 
\subsubsection{Global inference: Sampling the $\pi_i$}
Under the auxiliary variable scheme, each cluster is associated with a single processor. We jointly resample the processor allocations of all data points within a given cluster, allowing us to move an entire cluster from one processor to another. We use a Metropolis Hastings step with a proposal distribution independently assigns cluster $k$ to processor $j$ with probability $1/P$. This means our accept/reject probability depends only on the ratio of the likelihoods of the two assignments. 

The likelihood ratio is given by:
\begin{equation}
\begin{split}
\frac{p(\{\pi_i^*\})}{p(\{\pi_i\})} =& \frac{p(\{x_i\}|\pi_i^*) p(\{\pi_i^*\}|\alpha,P)}{p(\{x_i\}|\pi_i) p(\{\pi_i\}|\alpha,P)} \\
=&\frac{p(\{z_i\}|\pi_i^*) p(\{\pi_i^*\}|\alpha,P)}{p(\{z_i\}|\pi_i) p(\{\pi_i\}|\alpha,P)}\\
=& \prod_{j=1}^{P}\prod_{i=1}^{\max(N_j,N_j^*)}\frac{a_{ij}!}{a_{ij}^*!},\label{eq:lhood}
\end{split}
\end{equation}
where $N_j$ is the number of data points on processor $j$, and $a_{ij}$ is the number of clusters of size $i$ on processor $j$. In fact, we can simplify Eq.~\ref{eq:lhood} further, since many of the terms in the ratio of factorials will cancel. A derivation of Eq.~\ref{eq:lhood} is given in Appendix~\ref{app:MHDP}.

The reassignment of clusters can be implemented in a number of different manners. Actually transferring data from one processor to another will lead to bottlenecks, but may be appropriate if the entire data set is too large to be stored in memory on a single machine. If we can store a copy of the dataset on each machine, or we are using multiple cores on a single machine, we can simply transfer updates to lists of which data points belong to which cluster on which machine. We note that the reassignments need not occur at the same time, reducing the bandwidth required.


\subsection{Parallel inference in the HDP}

Again, we can assign tokens $x_{mi}$ to one of $P$ processors according to $\pi_{mi}$. Conditioned on the processor assignment and the values of $\zeta_j$, the data on each processor is distributed according to a HDP. We instantiate the processor allocations $\pi_{mi}$ and the bottom-level DP parameters, plus sufficient representation to perform inference in the processor-specific HDPs. 
\subsubsection{Local inference: Sampling the table and dish allocations}
Conditioned on the processor assignments, we simply have $P$ independent HDPs, and can use any existing inference algorithm for the HDP. In our experiments, we used the Chinese restaurant franchise sampling scheme \cite{Teh:Jordan:Beal:Blei:2006}.
\subsubsection{Global inference: Sampling the $\pi_{mj}$ and the $\zeta_j$}
We can represent the $\zeta_j$ as $\zeta_j:=\gamma \xi_j$, where $\gamma \sim \Gam(\alpha,1)$ and $\boldsymbol{\xi}:=(\xi_1,\dots,\xi_P)\sim \Dir(\alpha / P,\dots,\alpha / P)$. We sample the $\pi_{mj}$ and the $\xi_j$ jointly, and then sample $\gamma$, in order to improve the acceptance ratio of our Metropolis-Hastings steps.

Again, we want to reallocate whole clusters rather than independently reallocate individual tokens. So, our proposal distribution again assigns cluster $k$ to processor $j$ with probability $1/P$. Note that this means that a single data point does not necessarily reside on a single processor -- its tokens may be split among multiple processors. We also propose $\boldsymbol{\xi}^*\sim\Dir(\alpha / P,\dots,\alpha / P)$, and accept the resulting state with probability $\min(1, r)$, where
\begin{equation}
\begin{split}
r&=\frac{p(\{x_{mi}\}|\{\pi_{mi}^*\},\gamma,\boldsymbol{\xi}^*,\alpha,P)}{p(\{x_{mi}\}|\{\pi_{mi}\},\gamma,\boldsymbol{\xi},\alpha,P)}\frac{p(\{\pi_{mi}^*\}|\gamma,\boldsymbol{\xi}^*)}{p(\{\pi_{mi}\}|\gamma,\boldsymbol{\xi}^)}\frac{p(\boldsymbol{\xi}^*|\alpha,P)}{p(\boldsymbol{\xi}|\alpha,P)}\\
&=\prod_{j=1}^P\frac{(\xi_j^*)^{T_{\cdot j}^* + \alpha / P}}{(\xi_j)^{T_{\cdot j} + \alpha / P}}\frac{T_{\cdot j}^*!}{T_{\cdot j}!}\frac{\Gamma(\alpha / P +T_{\cdot j})}{\Gamma(\alpha / P +T_{\cdot j}^*)}\prod_{i=1}^{n_{\cdots}}\frac{b_{ji}!}{b_{ji}^*!}\prod_{m=1}^M\frac{a_{jmi}!}{a_{jmi}^*!}.
\end{split}\label{eq:lhoodHDP}
\end{equation}
A derivation of Eq.~\ref{eq:lhoodHDP} is given in Appendix~\ref{app:MHHDP}. As before, many of the ratios can be simplified further, reducing computational costs.

As with the Dirichlet process sampler, we can either transfer the data between machines, or simply update lists of which data points are ``active'' on each machine. We can resample $\gamma$ after sampling the $\xi_j$ and $\pi_{mj}$ using a standard Metropolis Hastings step (described in Appendix~\ref{app:MHgamma}).

\section{Experimental evaluation}

Our goal in this paper is to employ parallelization to speed up inference in the DP and HDP, without introducing approximations that might compromise the quality of our model estimate. To establish whether we have achieved that goal, we perform inference in both the DP and HDP using a number of inference methods, and look at appropriate measures of model quality as a function of inference time. This captures both the speed of the algorithms and the quality of the approximations obtained.

\subsection{Dirichlet process mixture of Gaussians}
We begin by evaluating the performance of the Dirichlet process sampler described in Section~\ref{sec:DPinf}. We generated a synthetic data set of one million data points from a mixture of univariate Gaussians. We used 50 components, each with mean distributed according to $\mbox{Uniform}(0,10)$ and fixed variance of $0.01$. A synthetic data set was chosen because it allows us to compare performance with ``ground truth''. We compared four inference algorithms:
\begin{itemize}
\item \textbf{Auxiliary variable parallel Gibbs sampler (AVparallel)} -- the model proposed in this paper, implemented in Java.
\item \textbf{Sequential Monte Carlo (SMC)} -- a basic sequential importance resampling algorithm, implemented in Java.
\item \textbf{Variational Bayes (VB)} -- the collapsed variational Bayesian algorithm described in \citet{Kurihara:Welling:Teh:2007}. We used an existing Matlab implementation\footnote{Code obtained from
https://sites.google.com/site/ kenichikurihara/academic-software/variational-dirichlet-process-gaussian-mixture-model}.
\item \textbf{Asynchronous parallel DP (Asynch)} -- we modified the Asynchronous sampler for the HDP \cite{Asuncion:Smyth:Welling:2008} to be applicable to the DP. We implemented the sampler in Java, using the settings described in \citet{Asuncion:Smyth:Welling:2008}.
\end{itemize}
\begin{figure*}[ht!]
\centering
\subfigure[F1 score against run time for AVparallel.]{
	\includegraphics[width=.95\columnwidth]{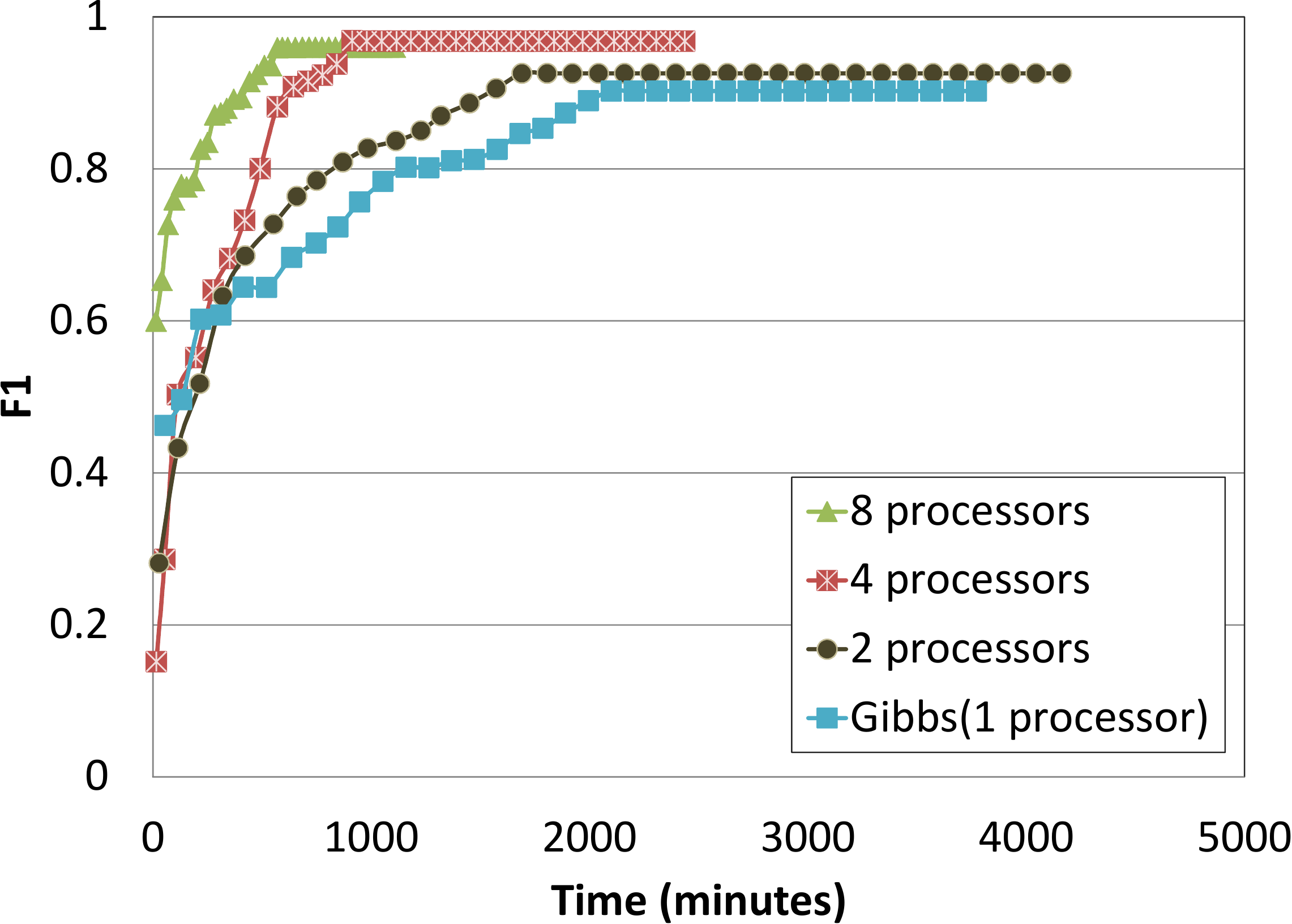}
	\label{fig:av_dp}
}
\subfigure[F1 score against run time for various algorithms. Unless otherwise specified, eight processors are used.]{
	\includegraphics[width=.95\columnwidth]{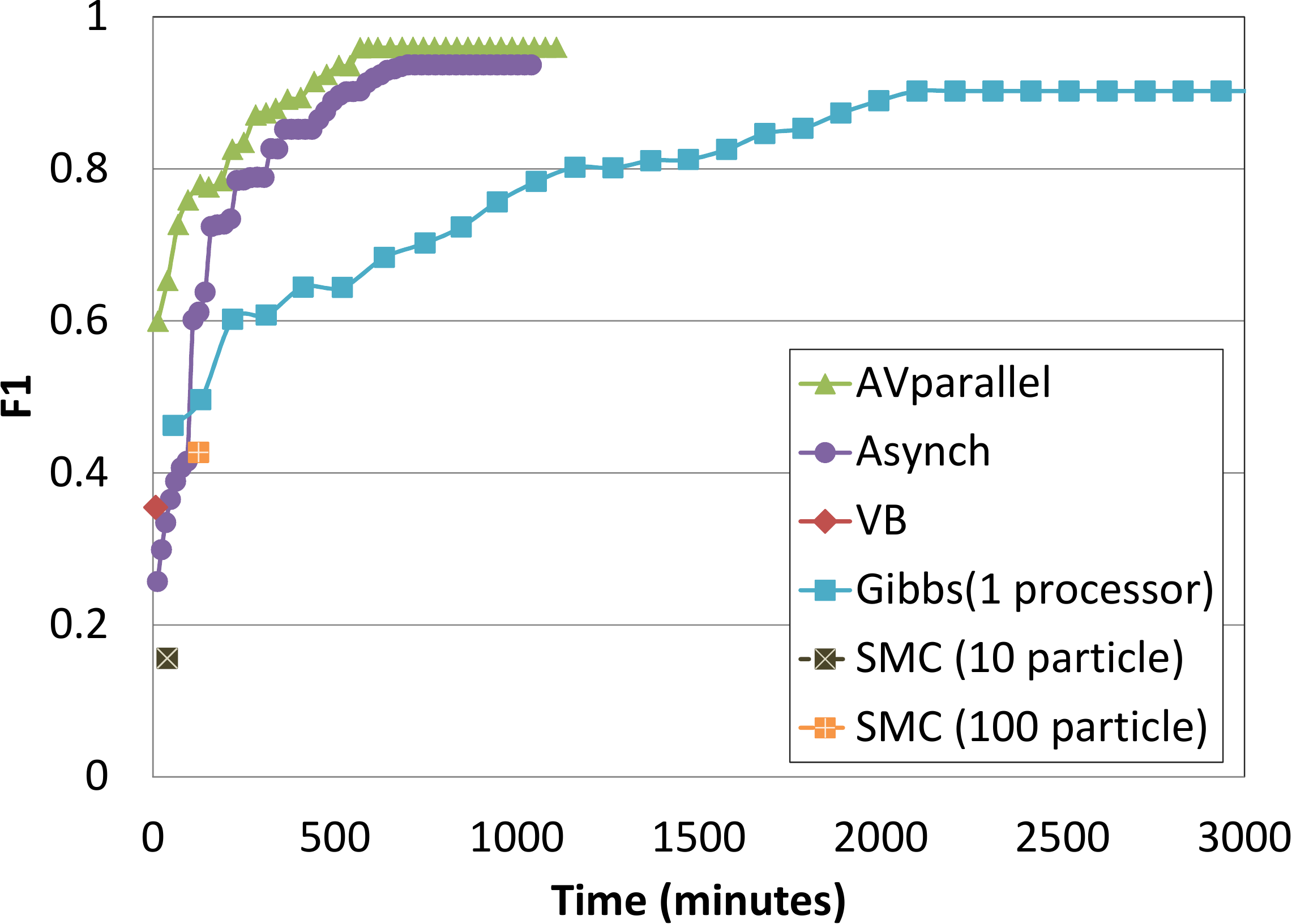}
	\label{fig:all_dp}
}
\subfigure[Time taken to reach convergence ($<0.1\%$ change in F1).]{
	\includegraphics[width=.95\columnwidth]{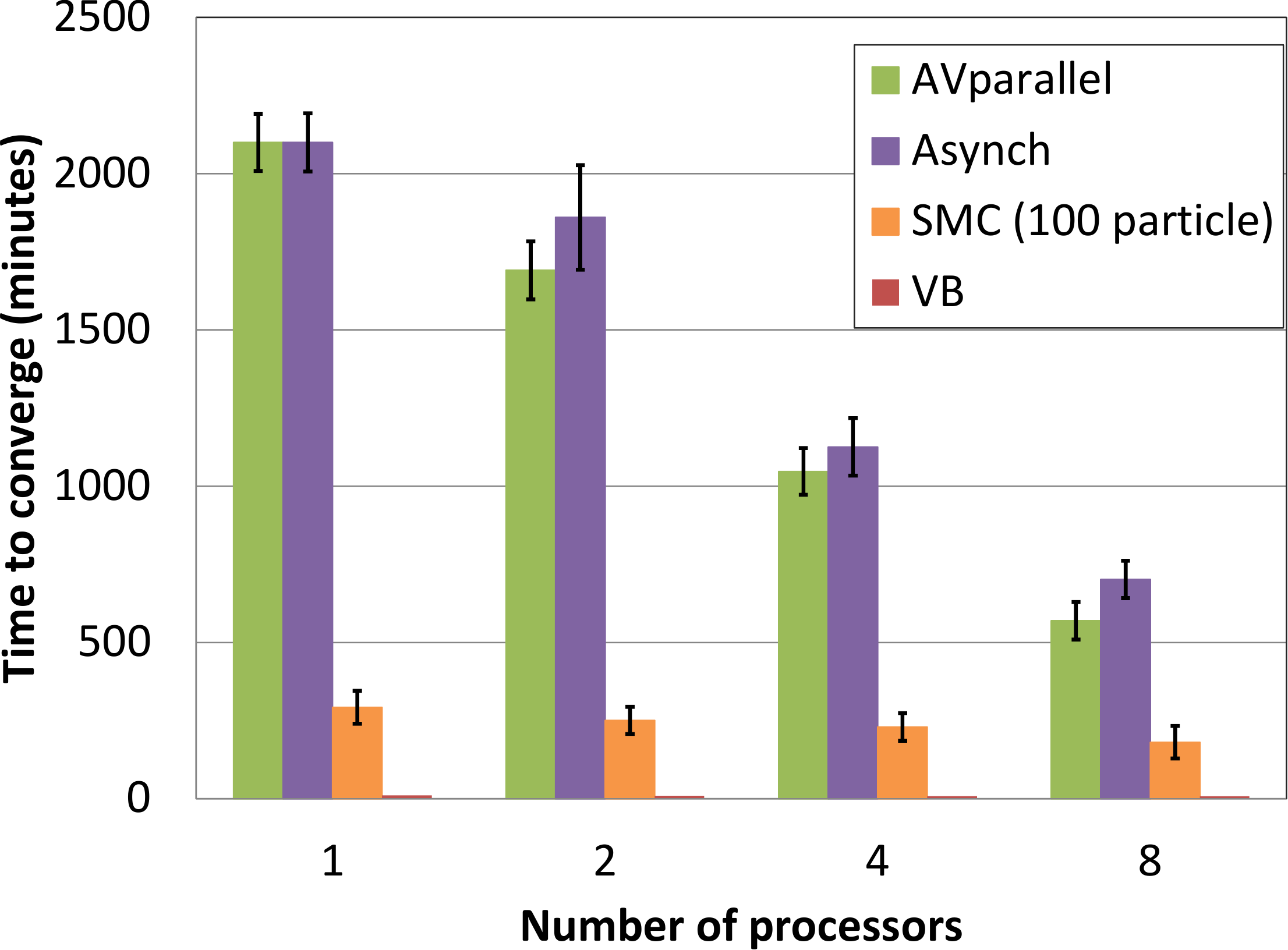}
	\label{fig:speed_dp}
}
\subfigure[Time spent in global and local steps for AVparallel, over 500 iterations.]{
	\includegraphics[width=.95\columnwidth]{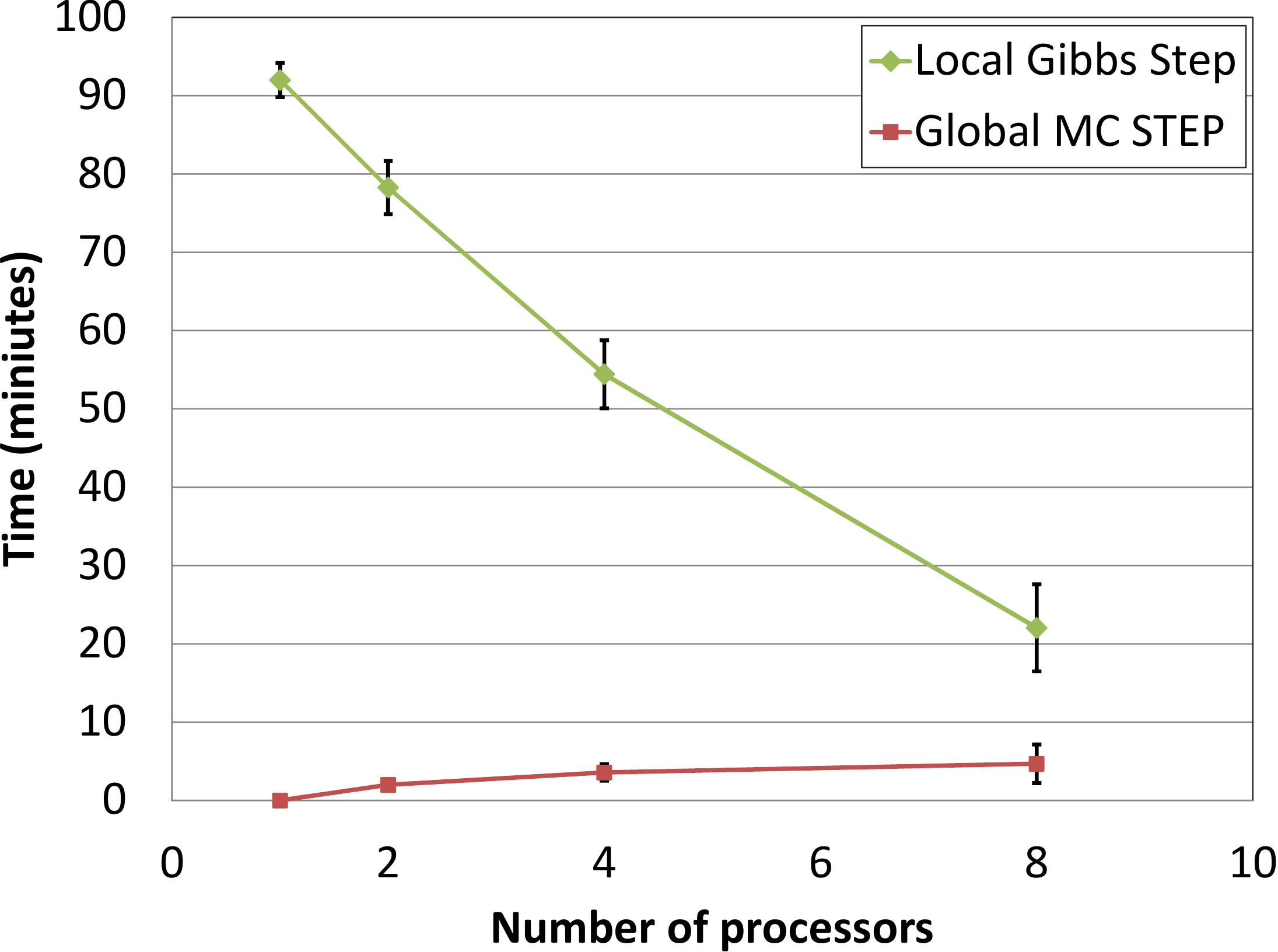}
	\label{fig:rel_dp}
}
\caption{Synthetic data modeled using a DPMM.}
\label{fig:synthetic}
\end{figure*}
In each case, we ran the algorithms on one, two, four and eight processors on a single multi-core machine\footnote{An AMD FX 8150 3.6 GHz (8 core) with 16 gig ram}, with one processor corresponding to the regular Gibbs sampler. We initialized each algorithm by clustering the data into 80 clusters using K-means clustering, and split the resulting clusters among the processors. 

We consider the F1 score between the clusterings obtained by each algorithm and the ground truth, as a function of time. Let $\mathcal{P}^{(g)}$ be the set of pairs of observations that are in the same cluster under ground truth, and let $\mathcal{P}^{(m)}$ be the set of pairs of observations that are in the same cluster in the inferred model. Then we define the F1 score of a model as the harmonic mean between the \emph{precision} -- the proportion $|\mathcal{P}^{(g)}\cap \mathcal{P}^{(m)}| /|\mathcal{P}^{(m)}|$ of pairs that are co-clustered by the model that also co-occur in the true partition -- and the \emph{recall} -- the proportion $|\mathcal{P}^{(g)}\cap \mathcal{P}^{(m)}| /|\mathcal{P}^{(g)}|$ of true co-occurrences that are co-clustered by the model. This definition of F1 is invariant to permutation, and so is appropriate in an unsupervised setting \cite{Xing:Ng:Jordan:Russell:2002}.

Figure~\ref{fig:av_dp} shows the F1 scores for our auxiliary variable method over time, using one, two, four and eight processors. As we can see, increasing the number of processors increases convergence speed without decreasing performance. Figure~\ref{fig:all_dp} shows the F1 scores over time for the four methods, each using eight cores. While we can get very fast results using variational Bayesian inference, the quality of the estimate is poor. Conversely, we achieve better performance (as measured by F1 score) than competing MCMC algorithms, with faster convergence. Figure~\ref{fig:speed_dp} shows the time taken by each algorithm to reach convergence, for varying numbers of processors. AV parallel and Asynch perform similarly. Figure~\ref{fig:rel_dp} shows the relative time spent sampling the processor allocations (the global step) and sampling the cluster allocations (the local step), over 500 iterations. This explains the similar scalability of AVparallel and Asynch: In AVparallel the majority of time is spent on local Gibbs sampling, which is implemented identically in both models. 




\begin{figure*}[ht!]
\centering
\subfigure[Test set perplexity against run time for AVparallel.]{
	\includegraphics[width=.95\columnwidth]{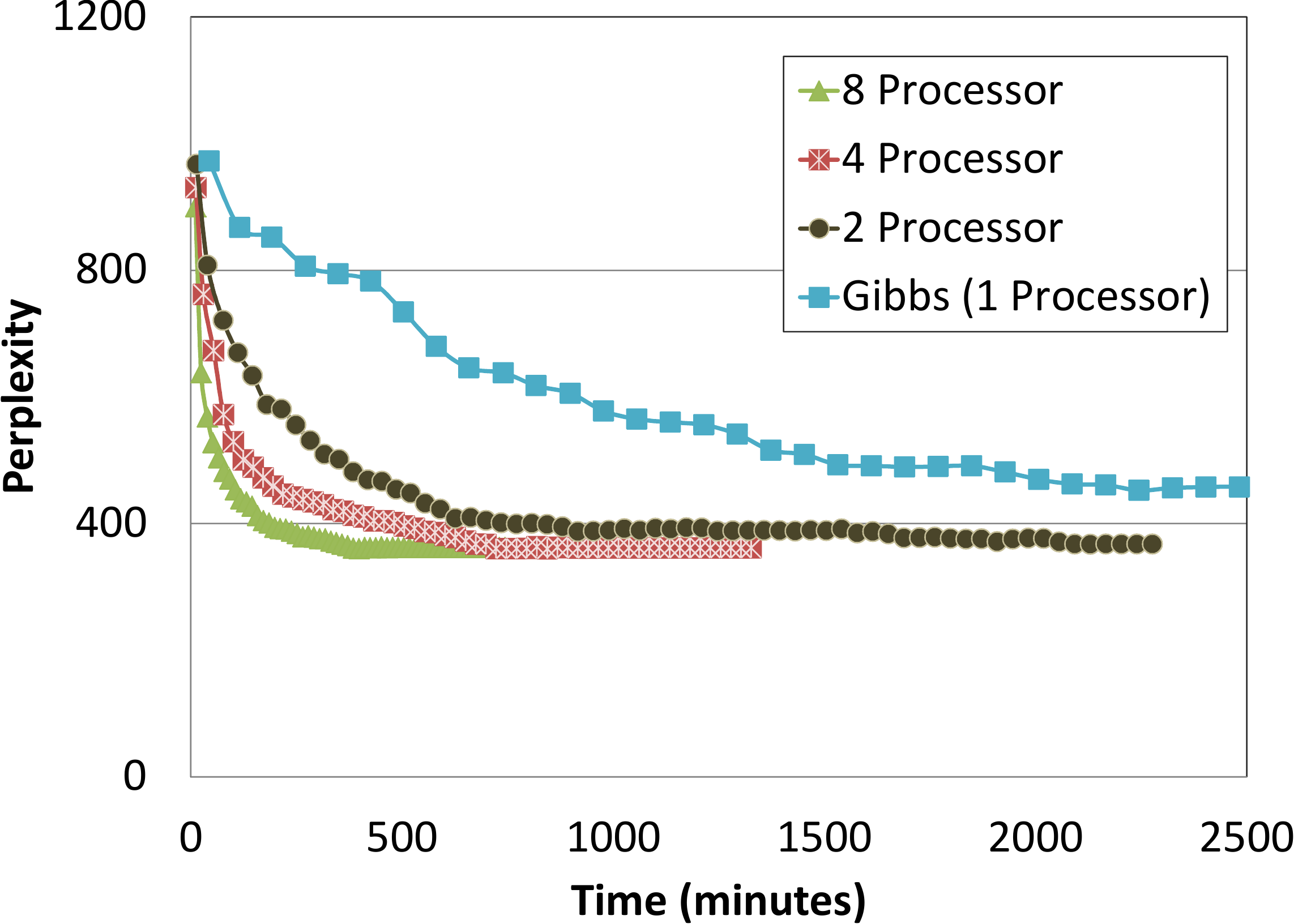}
	\label{fig:av_hdp}
}
\subfigure[Test set perplexity against run time for various algorithms. Unless otherwise specified, eight processors are used.]{
	\includegraphics[width=.95\columnwidth]{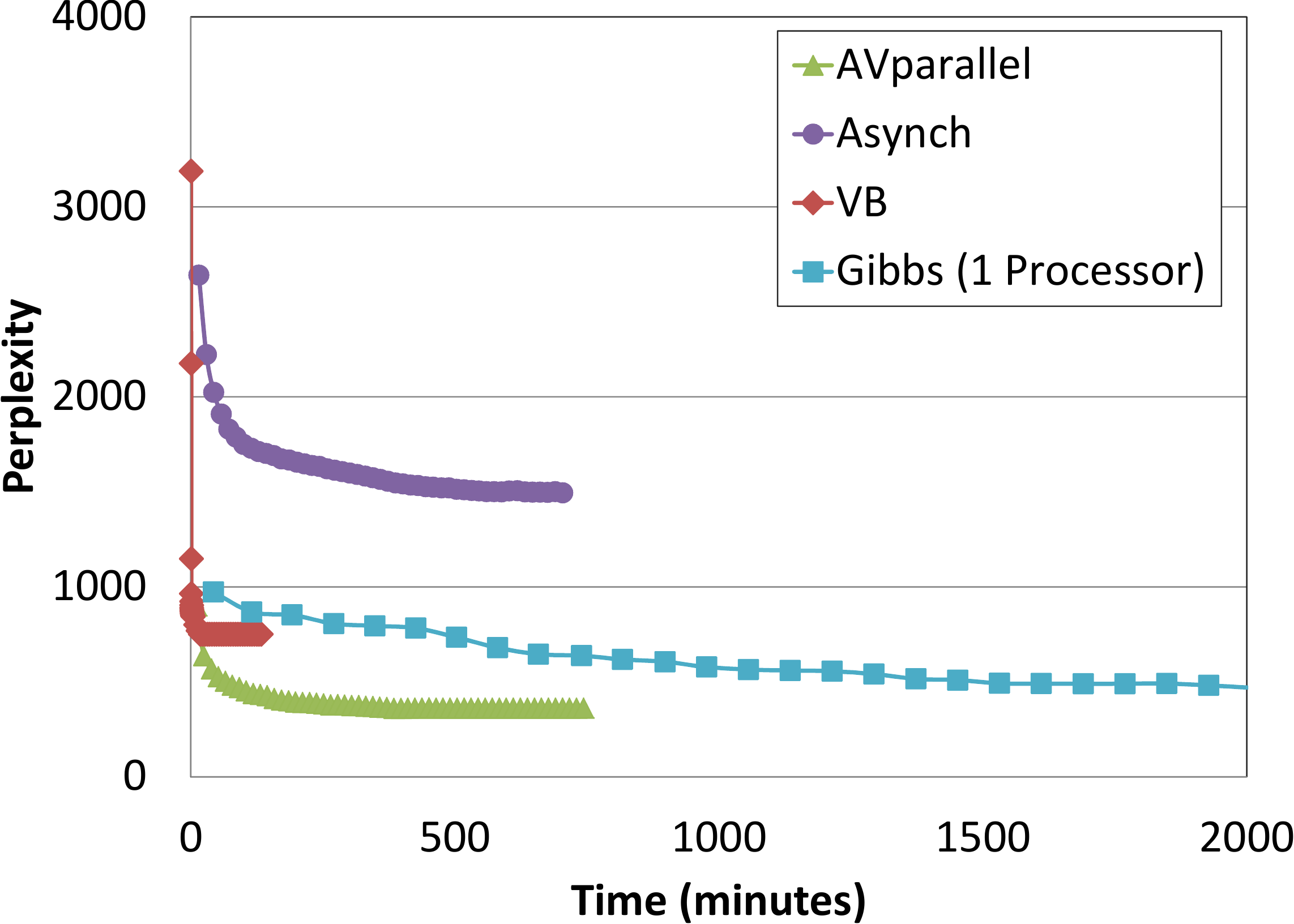}
	\label{fig:all_hdp}
}
\subfigure[Time taken to reach convergence ($<0.1\%$ change in perplexity).]{
	\includegraphics[width=.95\columnwidth]{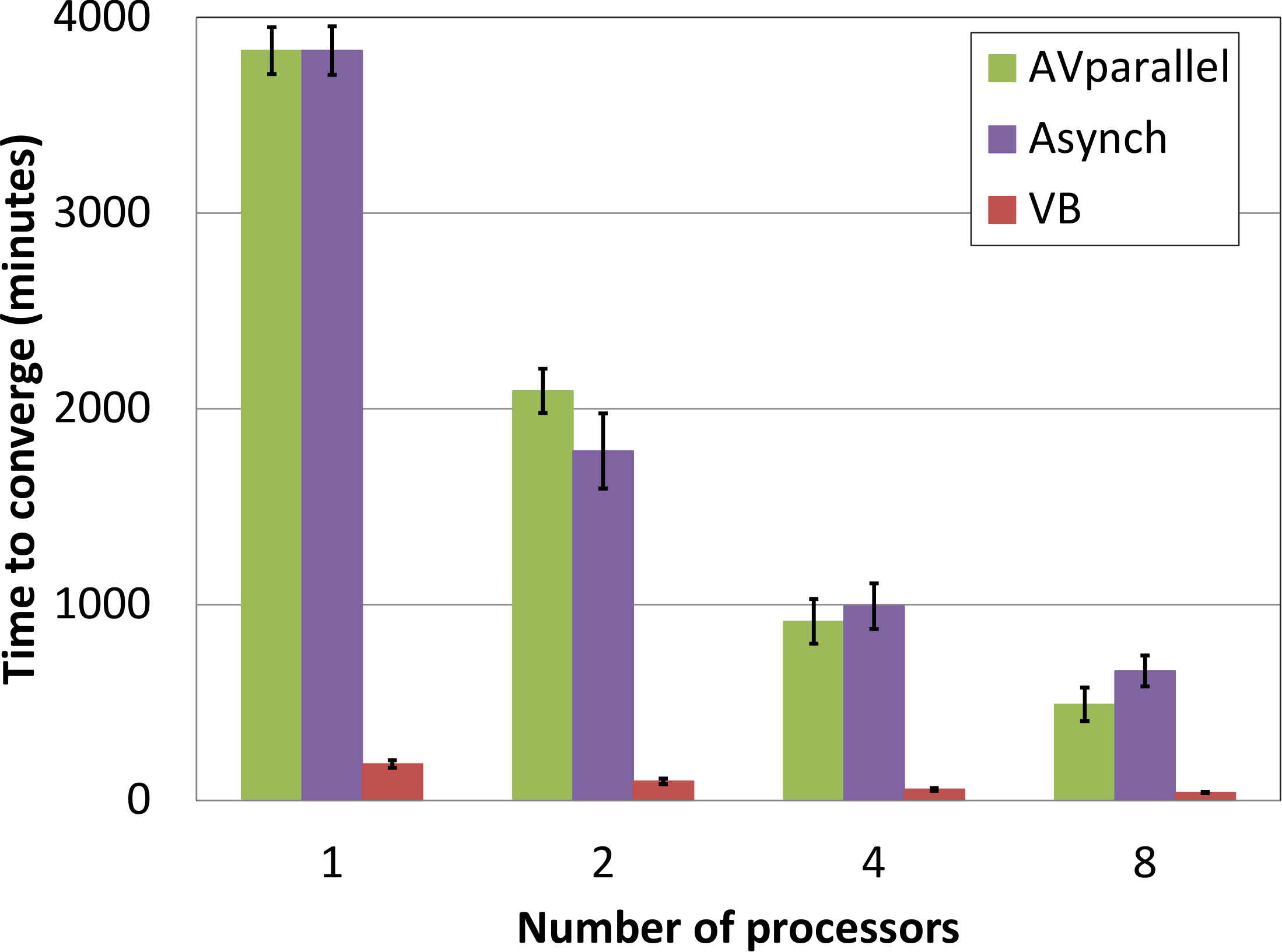}
	\label{fig:speed_hdp}
}
\subfigure[Time spent in global and local steps for AVparallel, over 20 iterations.]{
	\includegraphics[width=.95\columnwidth]{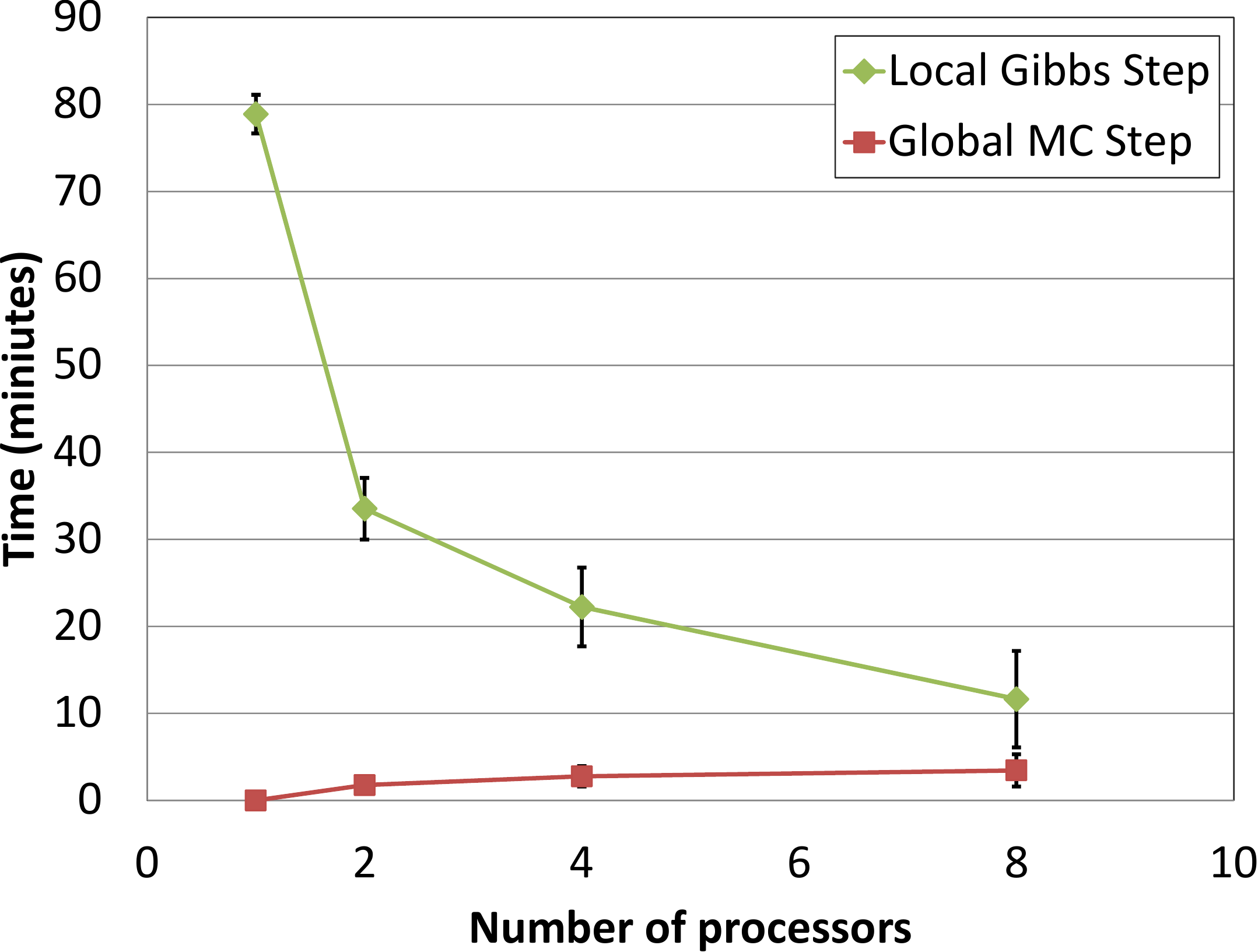}
	\label{fig:rel_hdp}
}
\caption{NIPS corpus, modeled using an HDP.}
\label{fig:NIPS}
\end{figure*}
\subsection{HDP topic model}

Next, we evaluate the performance of the HDP sampler on a topic modeling task as described by \citet{Teh:Jordan:Beal:Blei:2006}. Our dataset was a corpus of NIPS papers\footnote{http://ai.stanford.edu/~gal/data.html}, consisting of 2470 documents, containing 14300 unique words and around 3 million total words. We split the dataset into a training set of 2220 documents and a test set of 250 documents, and evaluated performance in terms of test set perplexity. We compared three inference methods:
\begin{itemize}
\item \textbf{Auxiliary variable parallel Gibbs sampler (AVparallel)} -- the model proposed in this paper, implemented in Java.
\item \textbf{Variational Bayes (VB)} -- the collapsed variational Bayesian algorithm described in \citet{Teh:Kurihari:Welling:2007}. We used an existing Java implementation\footnote{Code obtained from http://www.bradblock.com/tm-0.1.tar.gz}.
\item \textbf{Asynchronous parallel HDP (Asynch)} -- the asynchronous sampler for the HDP \cite{Asuncion:Smyth:Welling:2008}. We implemented the sampler in Java, using the settings described in the original paper.
\end{itemize}

Again, we ran each method on one, two, four and eight processors, and initialized each document to one of 80 clusters using K-means.

Figure~\ref{fig:av_hdp} shows the perplexity obtained using our auxiliary variable method over time, using one, two, four and eight processors. Figure~\ref{fig:all_hdp} compares the perplexities obtained using each of the three inference method over time, using eight processors. Figure \ref{fig:speed_hdp} shows the time to convergence of the three models for various numbers of processors.

As with the DPMM, while the variational approach is able to obtain results very quickly, the quality is much lower than that obtained using MCMC methods. The AVparallel method achieves much better perplexity than the approximate Asynch method -- the difference is much more striking than that seen in the DPMM. Note that, in the synthetic data used for the DPMM model, the true clusters are of similar size, while in the real-world data used for the HDP experiment there are likely to be many small clusters. We hypothesize that while the errors introduced in the asynchronous approximate method have little effect if the clusters are large, they become more significant if we have small clusters. Again, we find (Figure~\ref{fig:rel_hdp}) that the majority of time is spent in the local Gibbs sampler, meaning we can obtain a good rate of increase of speed by increasing the number of processors.

\section{Discussion and future work}
We have shown how alternative formulations for the DP and HDP can yield parallelizable Gibbs samplers that allow scalable and accurate inference. Our experiments show that the resulting algorithms offer better performance and scalability than existing parallel inference methods. 

The scalability of our algorithms is limited by the size of the largest cluster, since each cluster must reside on a single processor. An interesting avenue for future development is to investigate approximate methods for splitting large clusters onto multiple processors.
%

We have implemented and evaluated these methods on multi-core machines; our next goals are to develop and publish code appropriate for multi-machine architectures, and to extend our approach to other nonparametric models.

\bibliography{pDP}
\bibliographystyle{icml2013}
\appendix


\section{Theorem expanded proofs}\label{app:proofs}

\begin{theorem}[Auxiliary variable representation for the DPMM]
We can re-write the generative process for a DPMM as
\begin{equation}
\begin{split}
D_j\sim\DP\left(\frac{\alpha}{P}, H\right), \;\;\;\phi \sim \Dir\left(\frac{\alpha}{P},\dots,\frac{\alpha}{P}\right),\\
\pi_i\sim\phi,\qquad
\theta_i \sim D_{\pi_i}, \qquad
x_i\sim f(\theta_i),\qquad
\end{split}
\label{eq:DPaux}
\end{equation}
for $j=1,\dots,P$ and $i=1,\dots,N$.
The marginal distribution over the $x_i$ remains the same.\label{thm:DP}
\end{theorem}

\begin{proof}

In the main paper, we proved the general result, that if $\phi\sim\Dir(\alpha_1,...,\alpha_P)$ and $D_j \sim \DP(\alpha_j, H_j)$, then $D := \sum_j\phi_jD_j\sim \DP(\sum_j\alpha_j,  \frac{\sum_j\alpha_jH_j}{\sum_j\alpha_j})$. This result has been used by authors including \citet{Rao:Teh:2009}.

Here, we provide an explicit proof that shows the resulting predictive distribution is that of the Dirichlet process.

Let $\theta_1, \theta_2, \ldots $ be a sequence of random variable distributed according to $G \sim DP(\alpha,G_0)$. Then the conditional distribution of $\theta_{n+1}$ given $\theta_1, \ldots, \theta_{n}$ where $G$ has been integrated is given by

\begin{eqnarray}
\theta_{n+1} | \theta_1, \ldots, \theta_{n} \sim \sum_{l = 1}^{n} \frac{1}{n+\alpha} \delta_{\theta_l} + \frac{\alpha}{n  + \alpha} G_0
\end{eqnarray}

If $D_j \sim DP(\alpha/P, G_0)$, $\phi \sim Dir(\frac{\alpha}{P},\ldots, \frac{\alpha}{P})$, $\pi_i \sim \phi$ and $\theta_i \sim D_{\phi_i}$ then the conditional distribution of $\theta_{n+1}$ given $\theta_1, \ldots, \theta_{n}$ where $D_j, \forall j$ and $\phi$  has been integrated is given by 

\begin{equation}
\begin{split}
\theta_{n+1} |& \theta_1, \ldots, \theta_{n} \\\sim& \sum_{j = 1}^{P} P(\pi_{n+1} = j | \pi_1, \ldots, \pi_{n} ) \\
&P(\theta_{n+1} | \pi_{n+1}=j,\pi_1,\ldots,\pi_n, \theta_1, \ldots \theta_{n},G_0) \\
 =& \sum_j \frac{n_j + \alpha/P}{n+\alpha} \\
&\left \{ \sum_{l = 1}^{n} \frac{1}{n_j+\alpha/P} \delta_{\theta_l}\delta_{\pi_l = j} + \frac{\alpha/P}{n_j  + \alpha/P} G_0 \right \} \\
 =& \sum_{l = 1}^{n} \frac{1}{n+\alpha} \delta_{\theta_l} + \frac{\alpha}{n  + \alpha} G_0
\end{split}
\end{equation}
\end{proof}

\begin{theorem}[Auxiliary variable representation for the HDP]
If we incorporate the requirement that the concentration parameter $\gamma$ for the bottom level DPs $\{D_j\}_{j=1}^M$ depends on the concentration parameter $\alpha$ for the top level DP $D_0$ as $\gamma\sim\Gam(\alpha)$, then we can rewrite the generative process for the HDP as:
\begin{equation}
\begin{alignedat}{4}
\zeta_j &\sim \Gam(\alpha / P) &\qquad \qquad &\pi_{mi}\sim& \nu_m \quad \;\;\\
D_{0j} &\sim \DP(\alpha/P, H) &\qquad \qquad &\theta_{mi} \sim& D_{m\pi_{mi}}\\
\nu_m &\sim \Dir(\zeta_1,\dots,\zeta_P) &\qquad \qquad &x_{mi}\sim& f(\theta_{mi})\\
D_{mj} &\sim \DP(\zeta_j,D_{0j}), & & &
\label{eq:HDPaux}
\end{alignedat}
\end{equation}
for $j=1,\dots,P$, $m=1,\dots,M$, and $i= 1,\dots,N_m$. \label{thm:HDP}
\end{theorem}

\begin{proof}
Let $\zeta_j \sim \Gam(\alpha / P)$ and $D_{0j}\sim \DP(\alpha/P, H)$, $j=1,\dots,P$. This implies that $G_{0j}:=\zeta_jD_{0j} \sim \GaP((\alpha / P) H)$ and $\gamma:=\sum_{j=1}^P \zeta_j \sim \Gam(\alpha)$.

By superposition of gamma processes, 
\begin{equation*}
\begin{split}
G_0 &:= \sum_{j=1}^PG_{0j} \sim \GaP(\alpha H)\\
&:= \sum_{j=1}^P\zeta_jD_{0j}
\end{split}
\end{equation*}

Normalizing $G_0$, we get
\begin{equation*}
D_0(\cdot) := \frac{G_0(\cdot)}{\gamma} = \sum_{j=1}^P\frac{\zeta_j}{\gamma}D_{0j} \sim \DP(\alpha,H)
\end{equation*}
as required by the HDP.

Now, for $m=1,\dots, M$ and $j=1,\dots,P$, let $\eta_{mj}\sim \Gamma(\zeta_j)$ and $D_{mj}\sim \DP(\zeta_j, D_{0j})$. This implies that 
\begin{equation*}
\begin{split}
G_{mj}:=\eta_{mj}D_{mj} &\sim \GaP(\zeta_j D_{0j})\\
&:= \GaP(G_{0j}).
\end{split}
\end{equation*}
Superposition of the gamma processes gives
\begin{equation*}
\begin{split}
G_m := \sum_jG_{mj} &\sim \GaP(\sum_{j=1}^P G_{0j}) \\&:= \GaP(G_{0})
= \GaP(\gamma D_{0}).
\end{split}
\end{equation*}

The total mass of $G_m$ is given by $\sum_{j=1}^P\eta_{mj}$, so
\begin{equation}
D_m(\cdot) = \frac{G_m(\cdot)}{\sum_{j=1}^P\eta_{mj}} \sim \DP(\gamma, D_0)\label{eq:1}
\end{equation}
as required by the HDP.

If we let $\nu_{mj} = \eta_{mj} /sum_{k=1}^P\eta_{mk}$, then we can rewrite Equation~\ref{eq:1} as
\begin{equation}
D_m(\cdot) = \sum_{j=1}^P\nu_{mj}D_{mj} \sim \DP(\gamma, D_0),
\end{equation}
where $(\nu_{m1},\dots, \nu_{mP}) \sim \Dir(\zeta_1,\dots, \zeta_P)$.
\end{proof}

\section{Metropolis Hastings acceptance probabilities}\label{app:MH}
In both cases, the proposal probabilities $q(\{\pi_i\}\rightarrow \{\pi_i^*\}) = q(\{\pi_i^*\}\rightarrow \{\pi_i\})$, so we need only consider the likelihood ratios.

\subsection{Dirichlet process}\label{app:MHDP}

In the Dirichlet process case, the likelihood ratio is given by:

\begin{equation}
\begin{split}
\frac{p(\{\pi_i^*\})}{p(\{\pi_i\})} &= \frac{p(\{x_i\}|\pi_i^*) p(\{\pi_i^*\}|\alpha,P)}{p(\{x_i\}|\pi_i) p(\{\pi_i\}|\alpha,P)}\\
&= \frac{p(\{z_i\}|\pi_i^*) p(\{\pi_i^*\}|\alpha,P)}{p(\{z_i\}|\pi_i) p(\{\pi_i\}|\alpha,P)}\\
&= \prod_{j=1}^{P}\prod_{i=1}^{\max(N_j,N_j^*)}\frac{a_{ij}!}{a_{ij}^*!},
\end{split}\label{eq:lhood}
\end{equation}
where $N_j$ is the number of data points on processor $j$, and $a_{ij}$ is the number of clusters of size $i$ on processor $j$.


The probability of the processor allocations is described by the Dirichlet compound multinomial, or multivariate P\"{o}lya, distribution,
\begin{equation*}
\begin{split}
p(\{\pi_i\}|\alpha,\pi) =& \frac{N!}{\prod_{j=1}^{P}N_j!}\frac{\Gamma(\sum_{j=1}^P\alpha / P)}{\Gamma(N + \sum_{j=1}^P\alpha / P)}\\
&\prod_{j=1}^{P}\frac{\Gamma(N_j + \alpha / P)}{\Gamma(\alpha / P)}\\
=& \frac{N!}{\prod_{j=1}^{P}N_j!}\frac{\Gamma(\alpha)}{\Gamma(N + \alpha)}\prod_{j=1}^{P}\frac{\Gamma(N_j + \alpha / P)}{\Gamma(\alpha / P)},
\end{split}
\end{equation*}
where $N = \sum_{j=1}^PN_j$ is the total number of data points. So,
\begin{equation*}
\frac{p(\{\pi_i^*\}|\alpha,\pi)}{p(\{\pi_i\}|\alpha,\pi)}=\prod_{j=1}^P\frac{N_j}{N_j^*}\frac{\Gamma(N_j^*+\alpha / P)}{\Gamma(N_j + \alpha / P)}
\end{equation*}

Conditioned on the processor indicators, the probability of the data can be written
\begin{equation*}
p(\{z_i\}|\{\pi_i\}) =\prod_{j=1}^{P}p(\{n_{jk}\}|N_j),
\end{equation*}
where $n_{jk}$ is the number of data points in the $k$th on processor $j$. The distribution over cluster sizes in the Chinese restaurant process is described by Ewen's sampling formula, which gives:
\begin{equation*}
p(\{n_{jk}\}|N_j) =\bigg(\frac{\alpha}{P}\bigg)^{K_j} \frac{N_j!}{\prod_{k=1}^{K_j}n_{jk}!} \frac{\Gamma(\alpha / P)}{\Gamma(N_j + \alpha / P)}\prod_{i=1}^{N_j}\frac{1}{a_j!}
\end{equation*}
where $K_j$ is the total number of clusters on processor $j$. Therefore,
\begin{equation*}
\frac{p(\{n_{jk}^*\}|N_j)}{p(\{n_{jk}\}|N_j)} = \prod_j\frac{N_j^*!}{N_j!}\frac{\Gamma(N_j + \alpha / P)}{\Gamma(N_j^* + \alpha / P)}\prod_{1}^{\max(N_j,N_j^*)}\frac{a_{ij}!}{a_{ij}^*!},
\end{equation*}
so we get Equation~\ref{eq:lhood}.

\subsection{Hierarchical Dirichlet processes}\label{app:MHHDP}

For the HDP, the likelihood ratio is given by
\begin{equation*}
 \frac{p(\{x_{mi}\}|\{\pi_{mi}^*\,\gamma,\boldsymbol{\xi}^*,\alpha,P)}{p(\{x_{mi}\}|\{\pi_{mi}\,\gamma,\boldsymbol{\xi},\alpha,P)}\frac{p(\{\pi_{mi}^*\}|\gamma,\boldsymbol{\xi}^*)}{p(\{\pi_{mi}\}|\gamma,\boldsymbol{\xi}^)}\frac{p(\boldsymbol{\xi}^*|\alpha,P)}{p(\boldsymbol{\xi}|\alpha,P)}
\end{equation*}

The first term in the likelihood ratio is the ratio of the joint probabilities of the topic- and table-allocations in the local HDPs. This can be obtained by applying the Ewen's sampling formula to both top- and bottom-level DPs. Let $\mathbf{t}_{j}$ be the count vector for the top-level DP on processor $j$ -- in Chinese restaurant franchise terms, $t_{jd}$ is the number of tables on processor $j$ serving dish $d$. Let $\mathbf{n}_{jm}$ be the count vector for the $m$th bottom-level DP on processor $j$ -- in Chinese restaurant franchise terms, $n_{jmk}$ is the number of customers in the $m$th restaurant sat at the $k$th table of the $j$th processor. Let $T_{mj}$ be the total number of occupied tables from the $m$th restaurant on processor $j$, and let $D_j$ be the total number of unique dishes on processor $j$. Let $a_{jmi}$ be the total number of tables in restaurant $m$ on processor $j$ with exactly $i$ customers, and $b_{ji}$ be the total number of dishes on processor $j$ served at exactly $i$ tables. We use the notation $n_{jm\cdot} = \sum_k n_{jmk}$, $T_{\cdot j} = \sum_m T_{mj}$, etc. 

\begin{equation*}
\begin{split}
&p(\{n_{jmk}\} |\gamma, \boldsymbol{\xi})\\
 =& \prod_{m=1}^{M}\prod_{j=1}^P(\gamma\xi_j)^{T_{mj}}\frac{n_{jm\cdot}!}{\prod_{k=1}^{T_{mj}}n_{jmk}!}\frac{\Gamma(\gamma\xi_j)}{\Gamma(\gamma \xi_j + n_{jm\cdot})} \prod_{i=1}^{N_j}\frac{1}{a_{jmi}!},
\end{split}
\end{equation*}
and
\begin{equation*}
\begin{split}
&p(\{t_{jd}\} | \alpha, P)\\
 =& \prod_{j=1}^P \bigg(\frac{\alpha}{P}\bigg)^{D_j}\frac{T_{\cdot j}!}{\prod_{d=1}^{D_j}t_{jd}!}\frac{\Gamma(\alpha / P)}{\Gamma(\alpha / P + T_{\cdot j})}\prod_{i=1}^{T_{\cdot j}}\frac{1}{b_{ji}},
\end{split}
\end{equation*}

so
\begin{equation}
\begin{split}
 &\frac{p(\{t_{jd}^*\},\{n_{jmk}^*\}|rest)}{p(\{t_{jd}\},\{n_{jmk}\}|rest)}\\
=& 
\prod_{j=1}^{P}\frac{(\xi_j^*)^{T_{\cdot j}^*}}{(\xi_j)^{T_{\cdot j}}}\frac{T_{\cdot j}^* !}{T_{\cdot j} !}\frac{\Gamma(\alpha / P + T_{\cdot j})}{\Gamma(\alpha / P + T_{\cdot j}^*)}\bigg(\frac{\Gamma(\gamma \xi_j^*)}{\Gamma(\gamma \xi_j)}\bigg)^M\\
\textstyle &\bigg\{ \prod_{i=1}^{\max(T_{\cdot j},T_{\cdot j}^*)}\frac{b_{ji}!}{b_{ji}^*!}\bigg\} \prod_{m=1}^{M}\frac{n_{jm\cdot}^*!}{n_{jm\cdot}!}\frac{\Gamma(\gamma\xi_j + n_{jm\cdot})}{\Gamma(\gamma\xi_j^* + n_{jm\cdot}^*)}\\
&\prod_{i=1}^{\max(N_j,N_j^*)}\frac{a_{jmi}!}{a_{jmi}^*!}.
\end{split}\label{eq:HDP1}
\end{equation}

The probability of the processor assignments is given by:
\begin{equation*}
\begin{split}
p(\{\pi_{mi}\}|\gamma,\boldsymbol{\xi}) =& \prod_{m=1}^{M}\frac{n_{\cdot m \cdot}!}{\prod_{j=1}^P n_{jm\cdot}!}\frac{\Gamma(\gamma)}{\Gamma(n_{\cdot m \cdot} + \gamma)}\\
&\prod_{j=1}^P\frac{\Gamma(\gamma \xi_j + n_{jm\cdot})}{\Gamma(\gamma \xi_j)},
\end{split}
\end{equation*}
so the second term is given by
\begin{equation}
\begin{split}
\frac{p(\{\pi_{mi}^*\}|\gamma,\boldsymbol{\xi})}{p(\{\pi_{mi}\}|\gamma,\boldsymbol{\xi})} =&  \prod_{j=1}^P \bigg(\frac{\Gamma(\gamma \xi_j)}{\Gamma(\gamma \xi_j^*)}\bigg)^M \\
&\prod_{m=1}^M\frac{n_{jm\cdot}!}{n_{jm\cdot}^*!}\frac{
\Gamma(\gamma \xi_j^* + n_{jm\cdot}^*)}{
\Gamma(\gamma \xi_j + n_{jm\cdot})}\label{eq:HDP2}
\end{split}
\end{equation}

The third term is given by
\begin{equation}
 \frac{p(\boldsymbol{\xi}^*|\alpha,P)}{p(\boldsymbol{\xi}|\alpha,P)} = \prod_{j=1}^P \bigg(\frac{\xi_j^*}{\xi_j}\bigg)^{\frac{\alpha}{P}}.\label{eq:HDP3}
\end{equation}

Combining Equations~\ref{eq:HDP1}, \ref{eq:HDP2} and \ref{eq:HDP3} gives an acceptance probability of $\min(1,r)$, where
\begin{equation}
r =  \prod_{j=1}^P\frac{(\xi_j^*)^{T_{\cdot j}^* + \alpha / P}}{(\xi_j)^{T_{\cdot j} + \alpha / P}}\frac{T_{\cdot j}^*!}{T_{\cdot j}!}\frac{\Gamma(\alpha / P +T_{\cdot j})}{\Gamma(\alpha / P +T_{\cdot j}^*)}\prod_{i=1}^{n_{\cdots}}\frac{b_{ji}!}{b_{ji}^*!}\prod_{m=1}^M\frac{a_{jmi}!}{a_{jmi}^*!}.
\end{equation}
\subsection{Sampling $\gamma$}\label{app:MHgamma}
We sample the HDP parameter $\gamma$ using reversible random walk Metropolis Hastings steps, giving an acceptance probability of
\begin{align*}
\min\bigg(1,{\left ( \frac{\gamma^*}{\gamma} \right )}^{T_{..}} 
{\left [ \frac{\Gamma(\gamma^*)}{\Gamma(\gamma)} \right ]}^M
\prod_{m=1}^M \frac{\Gamma(n_{.m.} + \gamma)}{\Gamma(n_{.m.} + \gamma^*)}
\bigg).
\end{align*}

\end{document}